\documentclass[11pt]{article}
\usepackage[margin=1in]{geometry}
\usepackage{amsmath,amssymb,amsthm}
\usepackage{booktabs}
\usepackage{graphicx}
\usepackage{xcolor}
\usepackage[numbers,sort&compress]{natbib}
\usepackage[hidelinks]{hyperref}
\usepackage{microtype}

\usepackage{algorithm}
\usepackage{algpseudocode}

\newcommand{\cstar}{C^{*}}

\newcommand{\phat}{\hat p}
\newcommand{\E}{\mathbb{E}}
\newcommand{\Prob}{\mathbb{P}}
\DeclareMathOperator{\Var}{Var}
\DeclareMathOperator*{\esssup}{ess\,sup}

\newtheorem{proposition}{Proposition}
\newtheorem{lemma}{Lemma}
\newtheorem{corollary}{Corollary}
\theoremstyle{definition}

\newtheorem{remark}{Remark}

\title{\bfseries A prior-free blind detection of information leakage from model predictions}

\author{%
\textbf{Laurence~A.~Jacobs}$^{1,2,*}$ \\[6pt]
{\small $^{1}$Center for Molecular Cardiology, University of Zurich, Zurich, Switzerland} \\
{\small $^{2}$Center for Complexity Sciences, National University of Mexico, Mexico City, Mexico} \\
{\small $^{*}$Correspondence: \href{mailto:laurence.jacobs@uzh.ch}{laurence.jacobs@uzh.ch}}}

\date{\today}

\begin{document}
\maketitle

\begin{abstract}
\noindent
Data leakage---the contamination of a predictive model with information unavailable
at baseline---is the dominant reproducibility failure in machine-learning-based
science, yet the detection tools in use require the training code, fresh external
data, or domain expertise. None operates on the artifact an auditor most often
holds: the model's output. We ask what can be decided about leakage from a list of
predictions and outcomes alone. We give a decision-theoretic framework in which
leakage diagnostics are functionals of the predicted-risk/outcome law, parameterized
by a threshold-weighting that places them in correspondence with proper scoring
rules and decision-curve analysis. Within it we prove a sharp impossibility:
a recalibrated leak that matches an honest model's calibration and discrimination is
indistinguishable from honest performance by \emph{any} function of the predictions,
so the broad class of leakage is detectable only against an externally supplied
ceiling on achievable discrimination. We then prove what leakage cannot hide: a
near-deterministic subgroup---the signature of a near-label leak---produces a
sustained unit-purity head that no legitimate predictor of a non-deterministic
outcome can manufacture, yielding a prior-free test. These results organize leakage
into a trichotomy---miscalibrated, broad-calibrated, and deterministic---each with a
matched detector and an explicit failure mode. We validate the trichotomy on UK
Biobank using time-windowed comorbidity leakage with known, graded severity,
measuring a detection floor of $\Delta\cstar \approx 0.007$ on this endpoint, below
which the residual leakage is both undetectable from output and too small to alter
conclusions. The numerical floor depends on cohort, prevalence, endpoint, and leakage mechanism; the structural lesson is general: output-only detection fails exactly where the residual leakage is indistinguishable from an honestly stronger predictor without an external benchmark.
The resulting test takes a prediction vector and returns a verdict in under a second
on commodity hardware.
\end{abstract}

\section{Introduction}

Leakage occurs when a model is trained on information that will not be legitimately
available when it is deployed, so that reported performance reflects a quantity the
model cannot reproduce in use \citep{kaufman2012,lones2024}. Recent surveys identify
it as a pervasive and field-spanning cause of irreproducible results, affecting
hundreds of published studies across many disciplines, and show that once leakage is
corrected, the apparent advantage of complex models over simple baselines often
disappears \citep{kapoor2023}. In one of the most extensively audited domains, of
$62$ COVID-19 imaging models reviewed by \citet{roberts2021}, none was found suitable
for clinical use, with leakage among the dominant failure modes. Leakage is therefore
not a niche modeling error but a systemic threat to the evidentiary value of
predictive claims.

The methods available to catch it operate upstream of the result. Static analysis of
the training pipeline detects mechanical leakage---train/test contamination,
preprocessing fit before splitting, repeated test reuse \citep{dwork2015}---but requires the source
code \citep{yang2022}. External validation exposes leakage as a
failure to replicate but requires an independent cohort, which most studies never
obtain. Risk-of-bias and reporting instruments such as PROBAST and TRIPOD catch leakage
through expert appraisal but are time-intensive and demand both subject and
methodological expertise \citep{wolff2019,reforms2024,collins2015,collins2024tripodai}. The artifact that an editor,
replicator, or internal auditor most often actually holds---the model's predictions
on an evaluation set, together with the realized outcomes---has no test of its own.

We ask a precise question: \emph{what can be decided about leakage from the pair
$(\phat, y)$ of predicted risks and outcomes, with nothing else?} The answer is
neither ``everything'' nor ``nothing,'' and the contribution of this paper is to
draw the boundary exactly and to supply the detectors that live on the good side of
it.

\paragraph{Contributions.}
(i) A decision-theoretic framework (Section~\ref{sec:framework}) in which output-level
leakage diagnostics are threshold-weighted net-benefit functionals, linked to proper
scoring rules through a mixture representation. (ii) An impossibility theorem
(Section~\ref{sec:impossibility}): calibrated leakage matched in discrimination is
invisible to any $(\phat,y)$ functional, so the broad class is detectable only
relative to an external discrimination ceiling. (iii) A prior-free positive result
(Section~\ref{sec:cannot-hide}): a sustained unit-purity head certifies leakage under
only the qualitative assumption that the outcome is not deterministic. (iv) The
resulting trichotomy (Section~\ref{sec:trichotomy}). (v) Validation on UK Biobank
time-windowed leakage with a measured sensitivity floor (Section~\ref{sec:results}).
(vi) A deployable sub-second test.

\section{Framework}
\label{sec:framework}

\paragraph{Setup.}
A model produces predicted risks $\phat_i \in [0,1]$ for $i=1,\dots,n$ with binary
outcomes $y_i \in \{0,1\}$. We treat $(\phat_i, y_i)$ as draws from a joint law $P$
on $[0,1]\times\{0,1\}$ with base rate $\pi = \E[y]$. A predictor is
\emph{calibrated} if $\E[y \mid \phat] = \phat$; this corresponds to moderate
calibration in the hierarchy of \citet{vancalster2016, vancalster2019}. Discrimination is reported as
$\cstar$\footnote{Throughout we write $\cstar$ for Harrell's concordance statistic \citep{harrell1996},
the probability that a randomly chosen case is ranked above a randomly chosen
control.}.

\paragraph{Net benefit and its threshold weighting.}
For a decision threshold $\tau$, the net benefit of acting on $\{\phat \ge \tau\}$ is
$\mathrm{NB}(\tau) = \frac1n\sum_i\!\big[y_i\mathbf 1\{\phat_i\!\ge\!\tau\}
- (1-y_i)\mathbf 1\{\phat_i\!\ge\!\tau\}\tfrac{\tau}{1-\tau}\big]$
\citep{vickers2006}, where $\tau/(1-\tau)$ is the harm-to-benefit exchange rate. Following the expected-net-benefit
framework \citep{jacobs_enb}, we weight net benefit over thresholds by a density $\eta$ on $(0,1)$,
\begin{equation}
\mathrm{ENB}_\eta = \int_0^1 \mathrm{NB}(\tau)\,\eta(\tau)\,d\tau
= \frac1n\sum_i\big[y_i H(\phat_i) - (1-y_i)G(\phat_i)\big],
\label{eq:enb}
\end{equation}
with $H(p)=\int_0^p\eta$ and $G(p)=\int_0^p \eta(\tau)\tfrac{\tau}{1-\tau}\,d\tau$.
By the Schervish mixture representation of proper scoring rules
\citep{schervish1989,ehm2016,gneiting2007}, $\eta$ is the mixing measure: each $\eta$ selects a
proper score, and steering its mass toward $\tau\!\to\!1$ produces a functional
sensitive to the most stringent operating points. Equation~\eqref{eq:enb} is thus a
\emph{family} of probes, not a single number. Throughout this paper we adopt the
uniform default $\eta\equiv 1$, which suffices to exhibit the three regimes;
applications with specific decision-threshold ranges (e.g.\ screening at low $\tau$,
treatment selection at high $\tau$) admit detector variants specialized to that
range without modifying the framework. This representation matters because leakage need not be global; it may appear only in clinically relevant threshold regions or in the extreme-risk head of the prediction distribution.

\paragraph{Two leakage observables.}
Writing $r_i = y_i - \phat_i$ and $w(p)=H(p)+G(p)$, \eqref{eq:enb} decomposes as
$\mathrm{ENB}_\eta = \widetilde B_\eta + \frac1n\sum_i r_i w(\phat_i)$, where
$\widetilde B_\eta$ depends on $\phat$ alone. The weighted residual yields a
\emph{dispersion} statistic testing the calibrated null $y\mid\phat\sim
\mathrm{Bernoulli}(\phat)$,
\begin{equation}
V_\eta = \frac{\sum_i w(\phat_i)^2 r_i^2}{\sum_i w(\phat_i)^2\,\phat_i(1-\phat_i)},
\qquad \E_{\text{null}}[V_\eta]=1,\quad V_\eta = 1 + O_p(n^{-1/2}).
\label{eq:V}
\end{equation}
Separately, ordering predictions descending, the cumulative \emph{purity}
$\rho(k)=k^{-1}\sum_{i\le k} y_{(i)}$ (top-$k$ event rate) summarizes the head of the
risk distribution. We distinguish its \emph{breadth} (the largest top fraction with
$\rho - \rho_{\text{ref}} > \epsilon$) from its \emph{spike} (the largest top
fraction with absolute $\rho \ge 1-\delta$), as Sections~\ref{sec:impossibility}
and~\ref{sec:cannot-hide} show these measure different things.

\section{The impossibility result}
\label{sec:impossibility}

Leakage is defined by \emph{legitimacy}: a feature is illegitimate if it carries
information about $y$ that was not available at the prediction time
\citep{kaufman2012}. Legitimacy is a property of the temporal/causal structure of the
data-generating process---it is exogenous to the law $P$ of $(\phat,y)$. This is the
root of the limit. We make the limit visible through two short lemmas; the
impossibility proposition is then immediate.

\begin{lemma}[Determination]
\label{lem:determination-main}
A calibrated law $P$ on $[0,1]\times\{0,1\}$ is determined by its score marginal
$F=\mathrm{Law}(\phat)$:
\[
P(dp,\,y=1)=p\,F(dp),\qquad P(dp,\,y=0)=(1-p)\,F(dp),
\]
so $\mathrm{NB}(\tau)$, $\mathrm{ENB}_\eta$, and $\cstar$ are functionals of $F$
alone.
\end{lemma}

\begin{lemma}[Honest-world realizability]
\label{lem:realizable}
For every distribution $F$ on $[0,1]$ there exists an \emph{honest} predictor whose
induced law equals the calibrated law $P(F)$ of Lemma~\ref{lem:determination-main}.
Concretely, let $X\sim F$ be a legitimately observed prediction-time covariate, draw
$y\mid X{=}p\sim\mathrm{Bernoulli}(p)$, and take $\phat^{\mathrm h}=X$.
\end{lemma}

Lemma \ref{lem:realizable} should not be read as claiming that every calibrated prediction law is achievable by an leaker-free model in the same scientific problem. It shows something sharper: the joint law of ($\hat p,y$) contains no record of the information set from which $\hat p$ was generated. Therefore, without external knowledge of the admissible prediction-time $\sigma$-algebra, the same output law is compatible with both a legitimate and an illegitimate generating mechanism.

\begin{proposition}[Invisibility of calibrated, matched leakage]
\label{prop:impossible}
Let an honest procedure and a leaky procedure induce laws $P_{\mathrm h}$ and
$P_{\mathrm l}$ on $[0,1]\times\{0,1\}$. Every output-level diagnostic is a
functional $T(P)$, so if $P_{\mathrm h}=P_{\mathrm l}$ no test based on $(\phat,y)$
can separate them. By Lemma~\ref{lem:determination-main}, two calibrated predictors
with the same score marginal share $P$ and are thus
output-indistinguishable. By Lemma~\ref{lem:realizable}, every calibrated leaky law
is the law of some honest predictor; calibrated, marginal-matched leakage is
therefore invisible to any function of the predictions.
\end{proposition}

Proofs of Lemmas~\ref{lem:determination-main}, \ref{lem:realizable} and the
proposition are in Appendix~\ref{app:proof}.

\begin{corollary}[The broad class needs a prior]
\label{cor:ceiling}
A calibrated leak whose only effect is to raise $\cstar$ cannot be flagged from
output alone, because its law coincides with that of a genuinely superior honest
model at the same $\cstar$. Detecting it requires an exogenous bound $\cstar_{\max}$
on the discrimination achievable without leakage ---an outcome-specific
prior.
\end{corollary}

Proposition~\ref{prop:impossible} is the floor under all output-level detection. It
is also liberating: it tells us precisely which leakage \emph{cannot} hide, namely
leakage that pushes $P$ outside the class of legitimately achievable laws in a way
certifiable without knowing that class exactly. The output-only regime is the same
one in which membership-inference attacks operate \citep{shokri2017}: a closely related
access regime, but opposite goal---we audit the procedure rather than the training set.

\section{What leakage cannot hide}
\label{sec:cannot-hide}

\begin{lemma}[Purity ceiling]
\label{lem:purity}
Suppose the outcome is not prediction-time-deterministic: there is no admissible
event $S$ with $\Prob(y=1\mid S)=1$ on a non-null set. Then for every legitimate
predictor the top-$k$ purity satisfies $\rho(k) < 1$ for every $k$ spanning a
non-null fraction, and a sustained unit-purity head of non-null width is impossible.
Consequently an observed unit-purity plateau over a non-null top fraction certifies
the use of information under which the outcome is (near-)deterministic---information
not legitimately available at prediction time---under the single qualitative prior
that the outcome is not deterministic.
\end{lemma}

\begin{proof}[Proof sketch]
A legitimate calibrated predictor's top-$k$ purity converges to the average of the
true risk over the selected set, which is bounded above by the supremum of the
admissible conditional risk; non-determinism makes this supremum $<1$ on any non-null
set. A near-label leak places a non-null subset at conditional risk $\to 1$, breaking
the bound. Appendix~\ref{app:proof}.
\end{proof}

Lemma~\ref{lem:purity} is the prior-free positive result: the \emph{spike} statistic
of Section~\ref{sec:framework}, unlike breadth, escapes Proposition~\ref{prop:impossible}
because the unit-purity law lies outside the honest class for any non-deterministic
outcome. Separately, miscalibration-inducing leakage leaves
$V_\eta\neq 1$ and is detectable prior-free via \eqref{eq:V}; but a leaker who
recalibrates returns $V_\eta\to 1$, so this signal, while free, is evadable.

\section{The trichotomy}
\label{sec:trichotomy}

Proposition~\ref{prop:impossible} and Lemma~\ref{lem:purity} partition leakage by
what it does to $P$ and therefore by what can detect it (Table~\ref{tab:trichotomy}).
\vspace{0.2cm}

\begin{table}[ht!]
\centering
\caption{The leakage trichotomy: each regime, its matched detector, and the prior it
requires. The fourth row is Proposition~\ref{prop:impossible} in force.}
\label{tab:trichotomy}
\vspace{0.2cm}
\begin{tabular}{@{}lll@{}}
\toprule
Regime & Detector & Prior required \\
\midrule
Miscalibrated & dispersion $V_\eta\neq1$ & none (but evadable by recalibration)\\
Broad calibrated       & breadth vs.\ ceiling      & outcome $\cstar_{\max}$ \\
Deterministic (near-label) & unit-purity spike     & only ``outcome not deterministic'' \\
\midrule
Smooth, sub-threshold  & --- (invisible)           & undetectable from output \\
\bottomrule
\end{tabular}
\end{table}

\noindent
The fourth row coincides empirically with the regime in which leakage is too small
to change conclusions (Section~\ref{sec:results}), so the limit of detectability and the limit of consequence arrive together.

\begin{corollary}[Closure under recalibration]
\label{cor:closure}
Post-hoc monotone recalibration of the predictions can move leakage between regimes
of Table~\ref{tab:trichotomy} but cannot exit all three. The dispersion signal is
destroyed ($z_V\to 0$), but the ranking---hence $\cstar$ and the unit-purity
head---is preserved. A leak that improved $\cstar$ remains caught by the ceiling
trigger of Corollary~\ref{cor:ceiling}; a leak that did not is by
Proposition~\ref{prop:impossible} output-indistinguishable from clean, and is by
construction without consequence at the population level.
\end{corollary}

\section{Methods}
\label{sec:methods}

\paragraph{Synthetic construction.}
To exercise each regime at controlled severity we generate a calibrated honest risk
$\phat=\sigma(s)$, $s\sim\mathcal N$, with $y\sim\mathrm{Bernoulli}(\phat)$, tuning
the spread to a target $\cstar$ ($n=200{,}000$, $\pi=0.02$, target $\cstar\approx0.84$).
We inject (a) miscalibrated leakage by applying a monotone logit-amplification
$\phat_{\text{mis}}=\sigma(\beta\cdot\mathrm{logit}(\phat))$ with $\beta=2.5$,
which preserves the ranking (hence $\cstar$) exactly while destroying calibration;
(b) broad calibrated leakage as the true posterior under a noisy proxy $z\mid y$,
isotonically recalibrated; (c) deterministic leakage as a calibrated near-label
flag on $0.3\%$ of the cohort. All arms share a common $\cstar\approx0.84$--$0.87$
so discrimination cannot explain any difference. A fifth arm applies Platt scaling \citep{platt1999}
to the miscalibrated predictions to test evadability of the dispersion signal.

\paragraph{UK Biobank cohort.}
We use UK Biobank \citep{sudlow2015} (Application 596880; $n=501{,}883$) with the incident
delirium endpoint (ICD-10 F05.x; prevalence $\approx 0.018$). Leakage is introduced in graded,
clinically interpretable form: eleven comorbidity flags are allowed to be populated
from progressively wider windows of hospital-episode data after baseline
($+1,+2,+3,+4,+5,+7,+10$ years, and full follow-up), so that each window yields an
out-of-fold (OOF) prediction vector with a known, monotone increase in leakage. Base
models are $\ell_1$-penalized logistic regression \citep{tibshirani1996} with 64 features
(53 clinical biomarkers and 11 comorbidity flags), five-fold stratified cross-validation,
and penalty $C=0.002$; predictions are out-of-fold.

\paragraph{Detector and thresholds.}
Detection operates on $(\phat,y)$ only and is independent of model fitting. We report
$\cstar$, $V_\eta$ \eqref{eq:V}, breadth, and spike. The deployable test (Algorithm~\ref{alg:detector}) returns
\textsc{leaky} if a unit-purity head of at least $k_{\min}$ cases and non-null width
is present (Lemma~\ref{lem:purity}), or if $\cstar$ exceeds a supplied
$\cstar_{\max}$ (Corollary~\ref{cor:ceiling}); a dispersion anomaly is reported as a
soft warning; otherwise \textsc{clean}, with the stated scope that smooth
sub-threshold leakage cannot be excluded.

\begin{algorithm}[ht!]
\caption{Blind leakage test: prediction vector $\to$ verdict.}
\label{alg:detector}
\begin{algorithmic}[1]
\Require predictions $\phat[1{:}n]$, outcomes $y[1{:}n]$; optional ceiling $\cstar_{\max}$
\Statex \textbf{Parameters:} $k_{\min}$, $\delta$, $\epsilon$, $z_\alpha$, weight $\eta$
\Ensure verdict $\in\{\textsc{leaky},\textsc{clean}\}$ with a soft dispersion flag
\State sort indices by descending $\phat$
\State $\rho(k)\gets k^{-1}\sum_{i\le k} y_{(i)}$ \Comment{cumulative top-$k$ purity}
\State $\mathrm{spike}\gets \max\{k:\rho(k)\ge 1-\delta\}$ \Comment{unit-purity head width}
\State compute $\cstar$, $V_\eta$, $z_V$ \eqref{eq:V}, and $\mathrm{breadth}$ at excess $\epsilon$
\State $\mathrm{warn}\gets(\lvert z_V\rvert>z_\alpha)$ \Comment{soft, recalibration-evadable}
\If{$\mathrm{spike}\ge k_{\min}$}
  \State \Return $(\textsc{leaky},\,\mathrm{warn})$ \Comment{near-label (Lem.~\ref{lem:purity})}
\EndIf
\If{$\cstar_{\max}$ given \textbf{and} $\cstar>\cstar_{\max}$}
  \State \Return $(\textsc{leaky},\,\mathrm{warn})$ \Comment{broad (Cor.~\ref{cor:ceiling})}
\EndIf
\State \Return $(\textsc{clean},\,\mathrm{warn})$ \Comment{smooth sub-threshold not excluded}
\end{algorithmic}
\end{algorithm}

\paragraph{Choice of $k_{\min}$.}
The Remark following Lemma~\ref{lem:purity} bounds the legitimate probability of a
unit-purity head of width $k$ by $M^{k}$, where $M=\esssup\mu$ is an assumed bound
on the admissible conditional risk. Inverting for a target false-certification rate
$\alpha$ gives
\[
k_{\min}\;\ge\;\lceil\log\alpha/\log M\rceil.
\]
Guaranteeing $\alpha=10^{-3}$ even under a pessimistic $M=0.9$ would require
$k_{\min}\ge 66$; we adopt the lighter $k_{\min}=50$, which yields a
false-certification budget of $0.9^{50}\!\approx\!5\!\times\!10^{-3}$ under that
conservative $M=0.9$ and $\approx\!10^{-15}$ under a typical $M=0.5$. The
remaining thresholds are set to spike purity $\rho\ge 0.95$ (admitting a small
near-deterministic fraction $\delta=0.05$ to avoid sensitivity to single mislabeled
events), breadth excess $\epsilon=0.01$ (one percentage-point separation from the
honest purity curve), and uniform $\eta\equiv1$.

\section{Results}
\label{sec:results}

\subsection{Synthetic validation of the trichotomy}
At matched $\cstar=0.842$, the dispersion statistic cleanly separates miscalibrated
leakage from clean performance ($z_{V}= 291$ vs.\ $-0.5$), while the
calibrated broad proxy is invisible to $V_\eta$ ($z_V\approx0.0$) and has net benefit
equal to clean at every threshold---confirming Proposition~\ref{prop:impossible}---
and is recovered only when an outcome ceiling is supplied. The deterministic arm is
caught by unit purity ($\rho(0.1\%)=1.000$) regardless of calibration. The
monotone-transform construction of the miscalibrated arm guarantees identical $\cstar$
to the clean baseline ($0.842$ for both), so the dispersion signal cannot be
attributed to improved discrimination. Applying Platt scaling to the miscalibrated
arm restores $z_V=-0.1$---indistinguishable from clean---while preserving
$\cstar=0.842$ and the purity profile exactly (Figure~\ref{fig:trichotomy}, dashed
red), confirming that the dispersion signal is evadable by post-hoc recalibration.

\begin{figure}[ht!]
\centering
\includegraphics[width=\textwidth]{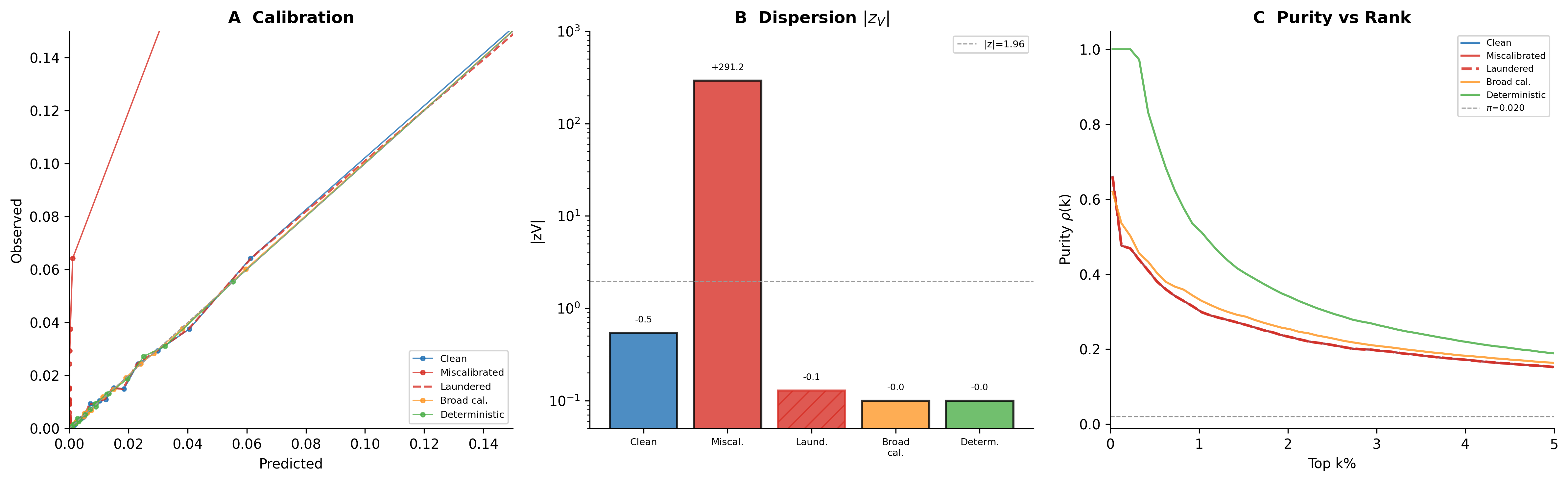}
\caption{Synthetic validation of the trichotomy at matched $\cstar\approx0.84$, $\pi=0.02$.
\textbf{A:}~Calibration catches the miscalibrated arm (red solid) but not the
laundered version (red dashed, Platt-scaled back to the diagonal); the broad calibrated (orange)
and deterministic (green) arms are also indistinguishable from clean.
\textbf{B:}~Dispersion $|z_V|$ on log scale: miscalibrated at $z_V=291$, but Platt scaling
restores $z_V\approx0$---the signal is evadable.
\textbf{C:}~Only the purity head catches the deterministic arm ($\rho=1.0$ at top $0.3\%$);
the laundered arm (dashed) overlaps clean exactly.}
\label{fig:trichotomy}
\end{figure}

\subsection{Graded leakage in UK Biobank and the detection floor}
The numerical floor measured here is specific to incident delirium at
$\pi\approx 0.018$ in UKB under $\ell_1$-penalized logistic regression with the
feature set of Section~\ref{sec:methods}; the structural claim---that the
output-level detection floor coincides with the conclusion-altering floor---is
general and follows from Proposition~\ref{prop:impossible}. With this scope fixed,
breadth rises monotonically with leakage and switches on at a sharp threshold
(Table~\ref{tab:windows}): zero through $+2$ years, then first nonzero at $+3$ years.
The detection floor is $\Delta\cstar \approx 0.007$--$0.008$; below it,
leakage leaves no output signature \emph{and} is too small to materially change
performance. Figure~\ref{fig:oc} shows the operating characteristic.

\begin{table}[ht!]
\centering
\caption{UK Biobank time-windowed leakage (incident delirium, ICD-10 F05.x). Breadth is the top
fraction with purity excess over honest exceeding $0.01$. $\Delta\cstar$ is computed from unrounded $\cstar$.}
\label{tab:windows}
\vspace{0.2cm}
\begin{tabular}{@{}lrrr@{}}
\toprule
Window & $\cstar$ & $\Delta\cstar$ & Breadth \\
\midrule
honest    & 0.7614 & +0.0000 & 0.0\% \\
$+1$y     & 0.7633 & +0.0019 & 0.0\% \\
$+2$y     & 0.7669 & +0.0056 & 0.0\% \\
$+3$y     & 0.7691 & +0.0077 & 1.2\% \\
$+4$y     & 0.7736 & +0.0122 & 1.8\% \\
$+5$y     & 0.7780 & +0.0167 & 4.2\% \\
$+7$y     & 0.7909 & +0.0295 & 11.4\% \\
$+10$y    & 0.8316 & +0.0702 & 22.8\% \\
full leak & 0.9249 & +0.1635 & 37.5\% \\
\end{tabular}
\end{table}

\subsection{The floor is intrinsic, not a regularization artifact}
A control sweeping the penalty across four orders of magnitude
($C\in[5\times10^{-4},1]$) leaves the $+2$y model with zero breadth and
$\Delta\cstar\approx0.005$--$0.006$ throughout: removing the penalty does not unmask a
hidden signal, because none exists. All eleven leaked comorbidity flags carry nonzero
coefficients at $+2$y across all penalty levels, confirming that the leak enters the
predictions; the floor is intrinsic to the information structure, consistent with
Proposition~\ref{prop:impossible}, not an artifact of regularization.

\subsection{The deterministic regime}
Isolating the dementia comorbidity flag as a near-label for incident delirium confirms
the prior-free spike. At full follow-up the flag coincides with the outcome (flagged
fraction $1.78\%\approx\pi$), giving $\cstar=1.000$ and a unit-purity head. The
instructive case is intermediate: at $+5$ years a deterministic head of only $0.16\%$
of the cohort is detected prior-free---a leak that raises $\cstar$ by just $0.024$ and
that a global metric would absorb as ordinary improvement, but that the unit-purity
test isolates. We also note that strong regularization can launder a near-label into
a broad lift, masking the spike; the prior-free guarantee assumes the model was
permitted to express what it learned (Section~\ref{sec:discussion}).

\subsection{False-positive rate on honest models}
\label{sec:fpr}
We apply the deployable test to ten leakage-free constructed models
spanning cardiometabolic, respiratory, neurological, oncologic, and mortality
domains, with no discrimination ceiling supplied so that only the prior-free spike
detector (Lemma~\ref{lem:purity}) is active. The cohort and feature set match
Section~\ref{sec:methods}; predictions are out-of-fold. All ten endpoints return
\textsc{clean}: no unit-purity head of width $\ge k_{\min}=50$ is present at the
$\rho\ge 0.95$ purity threshold in any vector, despite a broad spread of
discrimination ($\cstar$ ranging from 0.63 to 0.87). The prior-free false-positive
rate is $0/10$ (Table~\ref{tab:fpr}).

\begin{table}[ht!]
\centering
\caption{Blind leakage test applied to 10 leakage-free constructed
models spanning cardiometabolic, respiratory, neurological, oncologic, and
mortality domains.
The test operates on out-of-fold predictions with no discrimination ceiling
supplied, so only the prior-free spike detector (Lemma~\ref{lem:purity}) is active.
All 10 endpoints return \textsc{clean}; false positive rate = 0/10.}
\label{tab:fpr}
\vspace{0.2cm}
\begin{tabular}{@{}llrrrrl@{}}
\toprule
Code & Domain & Events & $\cstar$ & Spike@95\% & Breadth & Verdict \\
\midrule
DM2      & Cardiometabolic & 28{,}443 & 0.868 & 0     & ---  & \textsc{clean} \\
CHF      & Cardiometabolic &  8{,}271 & 0.761 & 0     & ---  & \textsc{clean} \\
CKD      & Cardiometabolic &  8{,}438 & 0.752 & 0     & ---  & \textsc{clean} \\
CHD      & Cardiometabolic & 24{,}209 & 0.690 & 0     & ---  & \textsc{clean} \\
COP      & Respiratory     & 10{,}332 & 0.848 & 0     & ---  & \textsc{clean} \\
DEM      & Neurological    &  5{,}568 & 0.770 & 0     & ---  & \textsc{clean} \\
CANlung  & Oncologic       &  4{,}820 & 0.787 & 0     & ---  & \textsc{clean} \\
CANclrc  & Oncologic       &  6{,}786 & 0.630 & 0     & ---  & \textsc{clean} \\
CANprost & Oncologic       & 12{,}081 & 0.658 & 0     & ---  & \textsc{clean} \\
ACM      & Mortality       & 55{,}023 & 0.712 & 0     & ---  & \textsc{clean} \\
\midrule
\multicolumn{6}{@{}l}{False positive rate} & 0/10 \\
\bottomrule
\end{tabular}
\end{table}

\subsection{Deployable test}
Because detection is a functional of $(\phat,y)$ and requires no model fitting, the
full battery runs in $1.3$\,s on the $\approx\!500$k-row cohort on commodity
hardware. On the four synthetic regimes it returns the expected verdicts:
\textsc{clean} for honest and for a strong calibrated model with no ceiling supplied,
\textsc{leaky} for the near-label prior-free, and \textsc{leaky} for the broad model
once a ceiling is supplied. On the honest-baseline panel of
Section~\ref{sec:fpr} the prior-free false-positive rate is $0/10$.

\begin{figure}[ht!]
\centering
\includegraphics[width=0.7\textwidth]{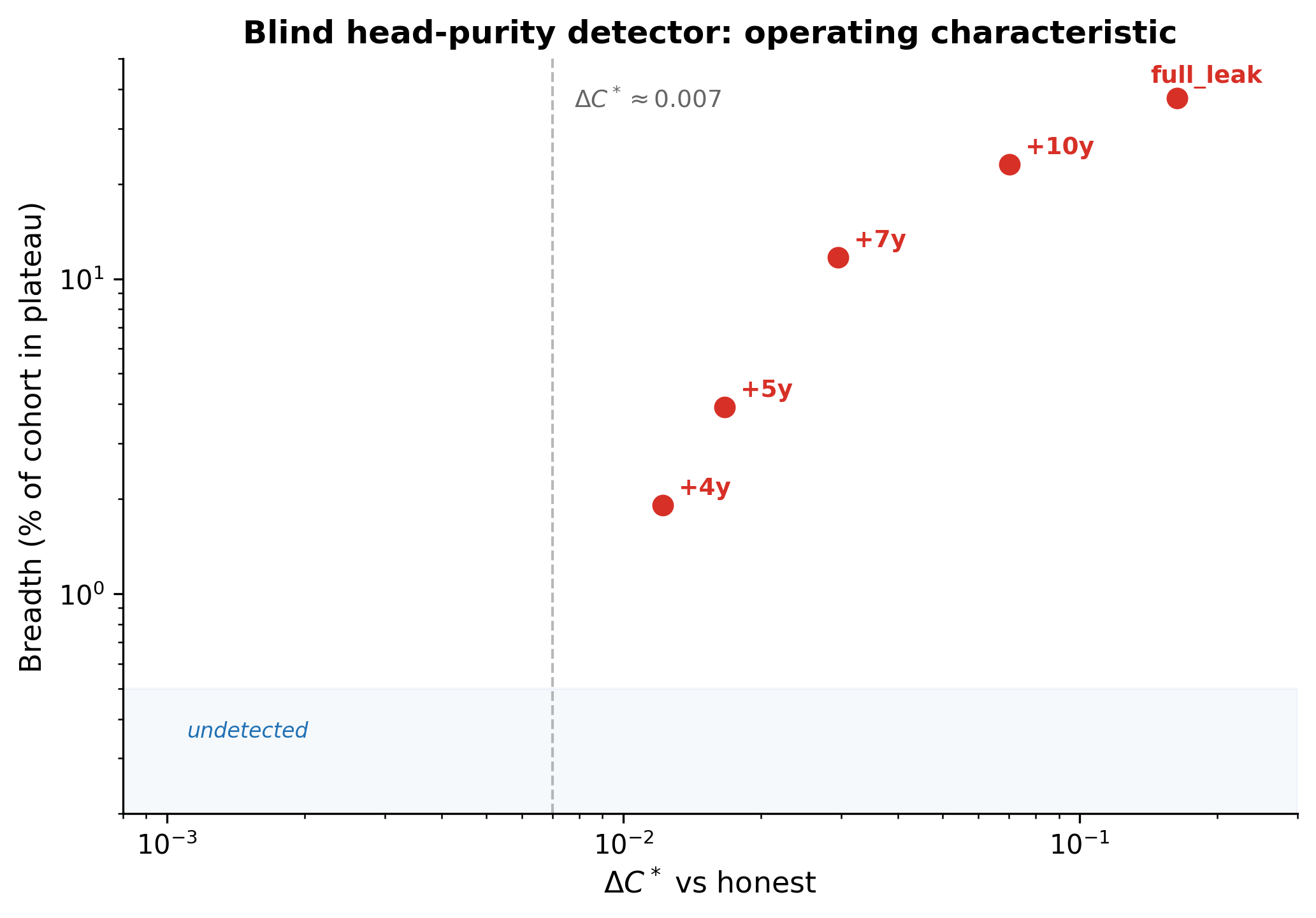}
\caption{Operating characteristic: detection (breadth / purity excess) versus
$\Delta\cstar$, with the floor at $\Delta\cstar\approx0.007$.}
\label{fig:oc}
\end{figure}

\section{Discussion}
\label{sec:discussion}

The detection literature splits into code-level static analysis and expert checklists,
with nothing operating on the model's output. This work supplies that missing layer
and bounds it exactly. The impossibility result is not a weakness to be apologized
for but the spine of the method: it tells a user precisely when each detector is
informative and when it is powerless.

\paragraph{In-house versus reviewer deployment.}
Corollary~\ref{cor:ceiling}'s prior dependence is a constraint for a reviewer handed
a single vector, but evaporates in a setting with honest baselines on record: passing
$\cstar_{\max}$ equal to one's own honest model plus a margin makes the broad trigger
automatic. The reviewer-facing version is the one genuinely bounded by
Proposition~\ref{prop:impossible}.

\paragraph{A standing screen.}
Because the cost of detection is negligible, leakage screening need not be a special
investigation; it can be a default column in a model registry, run on every endpoint
as a matter of course.

\paragraph{Recalibration and the trichotomy.}
Corollary~\ref{cor:closure} shows the trichotomy is closed under monotone
post-processing: a leaker who applies Platt scaling \citep{platt1999} moves from the
miscalibrated regime to the broad calibrated regime, destroying the dispersion
signal but leaving the ranking---and hence $\cstar$ and the unit-purity head---in
place (Figure~\ref{fig:trichotomy}, dashed red). The ceiling and spike triggers
therefore remain operative on recalibrated leakage; only a leak that does not
improve discrimination and is recalibrated is invisible, and is by
Proposition~\ref{prop:impossible} also without consequence.

\section{Limitations}
Smooth, calibrated, sub-threshold leakage is invisible to any output-level test
(Proposition~\ref{prop:impossible}); we prove this rather than work around it, and
show it coincides with the inconsequential regime. The broad trigger depends on an
outcome ceiling. The spike guarantee assumes the outcome is not
prediction-time-deterministic and that the model expresses the leaked information
rather than shrinking it away. The empirical validation uses a single cohort (UK
Biobank) with a specific endpoint and leakage mechanism; generalization to other
data structures (e.g., image-derived features, time-series) remains to be assessed.



\section*{Acknowledgments}

This research used the UK Biobank Resource under Application Number 596880. Data are available to approved researchers via \url{https://www.ukbiobank.ac.uk}~\cite{sudlow2015}. Top 25 ranked protein lists for both endpoints are provided in Supplementary Tables~S1 and S2. We thank all participants and the UK Biobank team for making this resource available.

\appendix
\section{Proofs}
\label{app:proof}

We work with the population law $P$ of $(\phat,y)$ on $[0,1]\times\{0,1\}$ and write
$F$ for its score marginal (the law of $\phat$), so $\pi=\E[y]=\int_0^1 p\,F(dp)$.

\begin{proof}[Proof of Lemma~\ref{lem:determination-main} (Determination)]
Calibration gives $\Prob(y=1\mid\phat=p)=p$ for $F$-a.e.\ $p$, which is the displayed
disintegration; hence $P$ is a measurable image of $F$. The population net benefit is
$\mathrm{NB}(\tau)=\int_{[\tau,1]}\big(p-\tfrac{\tau}{1-\tau}(1-p)\big)F(dp)$ and
$\mathrm{ENB}_\eta=\int_0^1\!\big(pH(p)-(1-p)G(p)\big)F(dp)$, both functionals of $F$.
Under calibration the case and control scores have laws $F_1(dp)=pF(dp)/\pi$ and
$F_0(dp)=(1-p)F(dp)/(1-\pi)$, so
\[
\cstar=\frac{1}{\pi(1-\pi)}\iint_{p>q} p\,(1-q)\,F(dp)\,F(dq)+\tfrac12(\text{ties}),
\]
again a functional of $F$ \citep{degroot1983,schervish1989}.
\end{proof}

\begin{proof}[Proof of Lemma~\ref{lem:realizable} (Realizability)]
Let $X$ be a covariate, defined on a richer space than the sample, with law
$\mathrm{Law}(X)=F$ on $[0,1]$; such a probability space always exists. Generate the
outcome $y$ conditionally on $X$ by $\Prob(y=1\mid X=p)=p$, and take the predictor
$\phat^{\mathrm h}=X$. The predictor is a function of the prediction-time covariate
$X$ alone, hence legitimate. Its induced joint law on $[0,1]\times\{0,1\}$ is the
calibrated law of marginal $F$ by Lemma~\ref{lem:determination-main}: $P_{\mathrm
h}=P(F)$. Moreover $\phat^{\mathrm h}=\Prob(y=1\mid X)$ is the Bayes risk relative to
the $\sigma$-algebra it generates, so among legitimate predictors with prediction-time
$\sigma$-algebra $\sigma(X)$ none has higher concordance.
\end{proof}

\begin{proof}[Proof of Proposition~\ref{prop:impossible}]
\emph{(i) No $(\phat,y)$-test separates equal laws.} Every output-level diagnostic is
a (possibly randomized) statistic of the sample $\{(\phat_i,y_i)\}$, whose sampling
law is determined by $P$. If $P_{\mathrm h}=P_{\mathrm l}$ the two procedures generate
identically distributed samples, so every statistic has the same law under both and
no test, of any size or power, can behave differently on them.

\emph{(ii) Matched leaky and honest laws coincide.} Suppose the leaky procedure is
calibrated with marginal $F$ and concordance $c=\cstar(F)$. By
Lemma~\ref{lem:realizable} there is an honest predictor with the same induced law
$P(F)$ and the same concordance $c$. By part~(i) the two procedures are
output-indistinguishable, although one uses only admissible information and the other
does not.

The mechanism is now explicit: legitimacy is a property of which $\sigma$-algebra
generated the prediction---the information available at prediction time
\citep{kaufman2012}---and is exogenous to $P$. Two procedures may share $P$ yet differ
in legitimacy, and no functional of $P$ can resolve that difference. Finally, a raw
leak that is miscalibrated has $\E[y\mid\phat]\neq\phat$ and is detectable through the
dispersion route (Appendix~\ref{app:dispersion}); but recalibration replaces it by a
calibrated law with some marginal $F'$, which by Lemma~\ref{lem:determination-main}
equals the honest law $P(F')$. Recalibration always maps a leak into the honest
class as a law, giving Corollary~\ref{cor:ceiling}: a calibrated leak whose only
effect is to raise $\cstar$ to $c'$ coincides with the honest model of marginal $F'$
at the same $c'$, so flagging it requires an exogenous ceiling $\cstar_{\max}$.
\end{proof}

For the purity ceiling, fix the sub-$\sigma$-algebra $\mathcal G$ of information
available at prediction time and let $\mu=\Prob(y=1\mid\mathcal G)=\E[y\mid\mathcal G]$
be the prediction-time risk; a predictor is legitimate iff it is $\mathcal
G$-measurable. For a top selection of population mass $u\in(0,1]$, let $\rho^\star(u)$
be the largest achievable top-$u$ event rate over legitimate predictors.

\begin{proof}[Proof of Lemma~\ref{lem:purity}]
For any $\mathcal G$-set $A$ with $\Prob(A)=u$,
\[
\E[y\,\mathbf 1_A]=\E[\E[y\mid\mathcal G]\mathbf 1_A]=\E[\mu\,\mathbf 1_A]
\le \E[\mu\,\mathbf 1\{\mu\ge q_u\}],
\]
where $q_u$ is the upper-$u$ quantile of $\mu$, since among $\mathcal G$-sets of mass
$u$ the integral of $\mu$ is maximized by the upper level set $\{\mu\ge q_u\}$
(Hardy--Littlewood; \citealp{lieb2001}, Ch.~3). Hence the best legitimate top-$u$ purity is
\[
\rho^\star(u)=\frac1u\,\E[\mu\,\mathbf 1\{\mu\ge q_u\}]=\E[\mu\mid\mu\ge q_u].
\]
Suppose the outcome is not prediction-time-deterministic: $\Prob(\mu=1)=0$, i.e.\ no
admissible event of positive mass has $y=1$ almost surely. Then for every $u>0$ the
event $\{\mu\ge q_u\}$ has mass $\ge u>0$ and carries positive mass where $\mu<1$, so
\[
\rho^\star(u)=\E[\mu\mid\mu\ge q_u]<1\qquad\text{for all }u\in(0,1],
\]
and no legitimate predictor attains top-$u$ purity $1$ on a non-null fraction: a
sustained unit-purity head of non-null width $u>0$ is impossible. If additionally the
conditional risk is uniformly bounded away from one, $M:=\esssup\mu<1$, the gap is
uniform, $1-\rho^\star(u)\ge 1-M>0$. Contrapositively, an observed $\rho(u)=1$ over a
non-null fraction forces $\Prob(\mu=1)>0$ on the selected set, i.e.\ the predictor used
information rendering $y$ (near-)deterministic there---information not legitimately
available at prediction time. The only prior invoked is the qualitative
$\Prob(\mu=1)=0$.
\end{proof}

\begin{remark}[Finite-sample $k_{\min}$]
With $n$ samples the empirical top-$k$ purity $\rho(k)=k^{-1}\sum_{i\le k}y_{(i)}$ can
equal $1$ for small $k$ even legitimately, as a chance run of cases at the head. Under
the legitimate optimum the top labels are Bernoulli with means $\le M$, so the
probability that the first $k$ selected are all cases is at most $M^{k}\le
e^{-k(1-M)}$. Requiring width $k\ge k_{\min}$ with $M^{k_{\min}}$ below the target
false-certification level turns the population statement into a finite-sample test;
this is the role of $k_{\min}$ in the deployable detector.
\end{remark}

\section{Dispersion null and finite-size behavior}
\label{app:dispersion}

The statistic $V_\eta$ sits in the dispersion-statistic lineage that begins with
goodness-of-fit tests for logistic regression \citep{hosmer1980} and probability
forecast assessment \citep{spiegelhalter1986}; the threshold-weighting $\eta$ extends
that family to decision-analytic stratifications. Write $a_i=w(\phat_i)^2$ and
$s_i^2=\phat_i(1-\phat_i)$, so
$V_\eta=\big(\sum_i a_i r_i^2\big)\big/\big(\sum_i a_i s_i^2\big)$ with
$r_i=y_i-\phat_i$.

\paragraph{Null mean.} Under the calibrated null $y_i\mid\phat_i\sim\mathrm{Bernoulli}
(\phat_i)$ we have $\E[r_i\mid\phat_i]=0$ and $\E[r_i^2\mid\phat_i]=s_i^2$. The
denominator depends on the scores alone, so
\[
\E_{\text{null}}[V_\eta\mid\phat]
=\frac{\sum_i a_i\,\E[r_i^2\mid\phat_i]}{\sum_i a_i s_i^2}
=\frac{\sum_i a_i s_i^2}{\sum_i a_i s_i^2}=1,
\]
hence $\E_{\text{null}}[V_\eta]=1$ exactly, not merely asymptotically.

\paragraph{Fluctuation scale.} For a Bernoulli$(p)$ residual $r^2\in\{(1-p)^2,p^2\}$,
and a direct computation gives $\E[r^2]=p(1-p)$ and
$\Var(r^2)=p(1-p)(1-2p)^2=s^2(1-2p)^2$. With
$V_\eta-1=\big(\sum_i a_i(r_i^2-s_i^2)\big)/\sum_i a_i s_i^2$, the numerator is a sum of
conditionally independent mean-zero terms with
\[
\Var_{\text{null}}\!\Big(\textstyle\sum_i a_i(r_i^2-s_i^2)\,\Big|\,\phat\Big)
=\sum_i a_i^2 s_i^2(1-2\phat_i)^2=O(n),
\]
while the denominator is $\sum_i a_i s_i^2=O(n)$. Thus
$\Var_{\text{null}}(V_\eta\mid\phat)=O(n)/O(n)^2=O(n^{-1})$ and
$V_\eta=1+O_p(n^{-1/2})$. The summands are bounded, so a Lindeberg CLT gives
\[
\sqrt n\,(V_\eta-1)\Rightarrow\mathcal N(0,\sigma_V^2),\qquad
\sigma_V^2=\lim_{n\to\infty}\frac{n\sum_i a_i^2 s_i^2(1-2\phat_i)^2}
{\big(\sum_i a_i s_i^2\big)^2},
\]
the reference law for the reported $z_V=(V_\eta-1)/\sqrt{\widehat{\Var}}$.

\paragraph{Behavior under miscalibration.} If leakage makes the predictor over- or
under-dispersed, $\E[r_i^2\mid\phat_i]=s_i^2+b(\phat_i)$ with calibration defect
$b\not\equiv0$. Then $\E[V_\eta]=1+\big(\sum_i a_i b(\phat_i)\big)/\sum_i a_i s_i^2$,
an $O(1)$ offset, so $z_V\asymp\sqrt n\to\infty$: the miscalibrated leak is detected, with
power growing in $n$. Recalibration sets $b\equiv0$ and restores
$\E[V_\eta]=1$---which is why this prior-free signal is nonetheless evadable.

\paragraph{Plateau versus smooth decay.} Let $\mu$ be the prediction-time risk of
Appendix~\ref{app:proof} and $\rho^\star(u)=\E[\mu\mid\mu\ge q_u]$ the legitimate
purity profile. On $u\in(0,1]$, $\rho^\star$ is continuous and nonincreasing with
$\rho^\star(1)=\pi$ and, under non-determinism, $\rho^\star(0^+)=\esssup\mu<1$; its
derivative $\tfrac{d}{du}\big(u\rho^\star(u)\big)=q_u$ (the $u$-quantile of $\mu$)
decays smoothly under any continuous risk distribution---the power-law-like lift decay
of an honest ranker. A near-label leak instead pins $\rho(u)\equiv1$ on the leaked
fraction $u\in(0,u_0]$: a flat plateau at the absolute ceiling, of non-null width,
which Appendix~\ref{app:proof} shows no legitimate predictor can produce. The
qualitative contrast---smooth sub-unit decay versus a unit-value plateau of non-null
width---is exactly the prior-free \emph{spike} signature, distinct from the
\emph{breadth} excess against a reference curve that Proposition~\ref{prop:impossible}
shows requires an external $\cstar_{\max}$.

\bibliographystyle{unsrtnat}
\bibliography{refs}

\end{document}